\documentclass[conference]{IEEEtran}
\ifCLASSINFOpdf
\else
\fi
\usepackage{url}

\PassOptionsToPackage{hyphens}{url}
\usepackage{makecell}
\usepackage{multirow}
\usepackage{amsthm,amsmath,amssymb,amsfonts}
\usepackage{mathtools}
\usepackage{graphicx}
\graphicspath{ {images/} }
\usepackage[font=normalsize]{subfig}
\usepackage{algorithm}
\usepackage[noend]{algpseudocode}
\algnewcommand{\LineComment}[1]{\State \(\triangleright\) #1}
\newcommand{\tabfiguresmall}[2]{\raisebox{-0.3\height}{\includegraphics[#1]{#2}}}

\theoremstyle{definition}
\newtheorem{definition}{Definition}
\newtheorem{theorem}{Theorem}

\newtheorem{lemma}{Lemma}

\def\x{{\bf x}}
\def\y{{\bf y}}
\def\p{{\bf p}}
\def\R{{\bf R}}

\def\boldr{{\bf r}}

\def\bolds{{\bf s}}
\def\boldt{{\bf t}}
\def\boldo{{\bf o}}
\def\boldO{{\bf O}}
\def\boldg{{\bf g}}
\def\boldV{{\bf V}}
\def\boldE{{\bf E}}
\def\boldhatO{{\bf \hat{O}}}
\def\0{{\bf 0}}
\def\rmmax{{\rm max}}
\def\argmax{\mathop{\rm argmax}}
\def\argwhere{\mathop{\rm argwhere}}

\usepackage{soul,xcolor}
\setstcolor{orange}

\newcommand{\HIGHLIGHT}[1]{\textcolor{red}{#1}}

\begin{document}
%
\title{Exploring Model Learning Heterogeneity\\for Boosting Ensemble Robustness}
\author{\IEEEauthorblockN{Yanzhao Wu$^{*,\dagger}$, Ka-Ho Chow$^{\dagger}$, Wenqi Wei$^{\dagger,\ddagger}$, Ling Liu$^{\dagger}$}
$^*$ Florida International University, Miami, FL 33199\\
$^\dagger$ Georgia Institute of Technology, Atlanta, GA 30332\\
$^\ddagger$ Fordham University, New York City, NY 10023\\
}


%


\maketitle

\begin{abstract}
Deep neural network ensembles hold the potential of improving generalization performance for complex learning tasks. This paper presents formal analysis and empirical evaluation to show that heterogeneous deep ensembles with high ensemble diversity can effectively leverage model learning heterogeneity to boost ensemble robustness. We first show that heterogeneous DNN models trained for solving the same learning problem, e.g., object detection, can significantly strengthen the mean average precision (mAP) through our weighted bounding box ensemble consensus method. Second, we further compose ensembles of heterogeneous models for solving different learning problems, e.g., object detection and semantic segmentation, by introducing the connected component labeling (CCL) based alignment. We show that this two-tier heterogeneity driven ensemble construction method can compose an ensemble team that promotes high ensemble diversity and low negative correlation among member models of the ensemble, strengthening ensemble robustness against both negative examples and adversarial attacks. Third, we provide a formal analysis of the ensemble robustness in terms of negative correlation. Extensive experiments validate the enhanced robustness of heterogeneous ensembles in both benign and adversarial settings. The source codes are available on GitHub at \url{https://github.com/git-disl/HeteRobust}.
\end{abstract}


%
\IEEEpeerreviewmaketitle

\section{Introduction}
\label{sec:intro}
Deep neural networks (DNNs) have achieved remarkable success in tackling numerous real-world data mining challenges. However, a well-trained DNN not only fails when tested on some unseen examples~\cite{out-of-distribution-baseline,out-of-distribution-detection,ensemble-tdsc}, but it is also vulnerable to adversarial attacks~\cite{adv-examples-seg-det,transferable-adv-attack-od,robust-adv-perturb-proposal-models,od-adv-lens-esorics}.
Adversarial robustness of deep learning against disruptive events is considered a prerequisite for many mission critical applications, e.g., medical diagnosis~\cite{adv-attacks-medical-ml,adv-attacks-medical-image-classification} and autonomous driving~\cite{robust-physical-world-attacks,adv-sensor-attack-LiDAR}.
Recent research on fusion learning has shown improved test accuracy on unseen examples that are captured under poor lighting, poor weather, or sensor noise conditions~\cite{deep-adaptive-fusion,seeing-through-fog}. However, such solutions remain vulnerable under adversarial attacks~\cite{adv-attacks-medical-ml,adv-attacks-medical-image-classification,robust-physical-world-attacks,adv-sensor-attack-LiDAR}.

\begin{figure*}[t!]
\centering
    \includegraphics[width=1.0\textwidth]{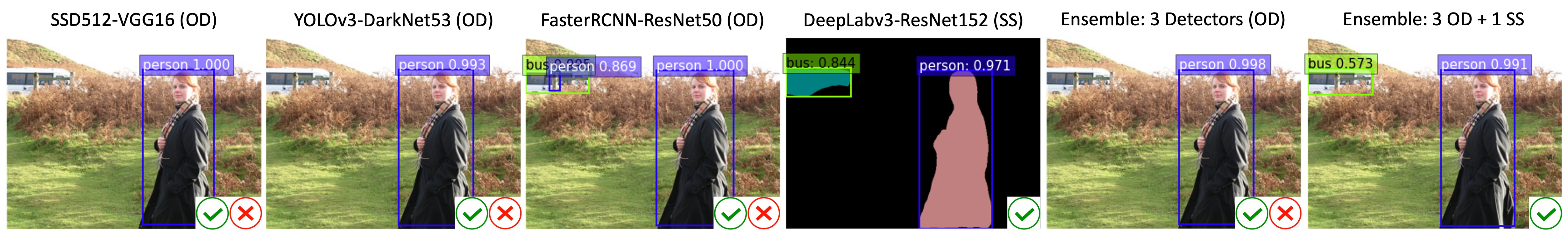}
    \caption{Object Detection (OD) and Semantic Segmentation (SS) Example on PASCAL VOC, the two-tier heterogeneous ensemble (3 OD + 1 SS) produces the correct predictions, outperforming the tier-1 heterogeneous ensemble (3 Detectors).}
    \label{fig:ensemble-seg-example-voc-intro}
\end{figure*}

\noindent {\bf Related Work and Problem Statement.}
Deep neural network ensembles hold the potential to provide strong adversarial robustness. Several orthogonal yet complementary studies in ensemble learning can be summarized into three broad categories. The first category has exploited the ensemble of multiple single task learners, represented by image classifiers, for improving the overall robustness against adversarial examples~\cite{ensemble-tdsc,ensemble-icnc,ensemble-bigdata,deepensembles,improving-adversarial-robustness}.
Some recent efforts have been devoted to developing the diversity metrics that can guide ensemble teaming such that a high quality ensemble team will be composed by the member classifiers with high diversity~\cite{ensemble-mass,EnsembleBenchCogMI,EnsembleBenchCVPR,EnsembleBenchICDM,error-diversity-object-detection-icdm}.
The second category has incorporated multiple tasks within a single model, such as different depth estimation learners in training a semantic segmentation model by using a shared encoder and performing multi-task learning at once~\cite{cross-stitch-net-multi-task,multi-task-nlu}. 
It is shown that multi-task learning can improve adversarial robustness of a single semantic segmentation model~\cite{multitask-robustness,robust-learning-cross-task,robustness-cross-domain-ensemble}.
The third category of efforts has explored multi-modal learning to train multi-modal networks with improved robustness by using data from multiple different sensors, such as cameras, LiDARs and Radars. The fusion of multiple different modalities shows the potential of improving prediction accuracy and robustness for vision tasks, e.g., object detection and semantic segmentation under extreme conditions~\cite{deep-adaptive-fusion,seeing-through-fog}.
However, there are still two critical challenges for achieving strong ensemble robustness: (1) how to properly measure and leverage ensemble diversity to compose robust ensemble teams of multi-task learners that are trained for solving the same problem, such as object detection; and (2) how to further strengthen ensemble robustness by performing ensemble consensus of multiple heterogeneous models that are trained for solving different learning problems, such as object detection and semantic segmentation.

\noindent {\bf Scope and Contributions.}
This paper addresses these problems of ensemble robustness with three original contributions.
{\it First}, we show that an ensemble of heterogeneous DNN models trained for solving the same learning problem, e.g., object detection, can strengthen the adversarial robustness compared to its best member model under adverse conditions, as long as the member models of the ensemble team are highly diverse and failure independent. One of the key challenges is to identify high quality ensembles that are composed of highly diverse member models with complementary prediction capacities. We leverage focal diversity metrics to precisely measure ensemble diversity and select high quality ensemble teams by their high focal diversity scores.
Another key challenge is to develop an effective detection consensus mechanism, which can harness the complementary wisdom of the diverse member models of an ensemble team to improve the overall mAP performance of the ensemble detector. We propose to integrate the bounding boxes using a weighted algorithm and combine the classification prediction by consensus voting.
{\it Second}, we explore the failure scenarios for the ensemble of multiple heterogeneous object detectors and develop additional optimization to further enhance the adversarial robustness of our heterogeneous model ensemble approach. Concretely, we introduce the second tier heterogeneous ensemble construction, which integrates heterogeneous models trained for solving different learning problems, say combining the tier-1 heterogeneous object detector ensemble with semantic segmentation models. We argue that this two-tier heterogeneity driven ensemble learning approach can provide high robustness guarantee and improve mAP under both negative examples and adversarial examples.
Figure~\ref{fig:ensemble-seg-example-voc-intro} shows a test example from PASCAL VOC. All three well trained object detectors fail partially: SSD512 and YOLOv3 fail to detect the \verb|bus| object, FasterRCNN mistakenly identifies a fabricated \verb|person| object next to the \verb|bus|, and the ensemble of these three heterogeneous detectors also fails to detect the \verb|bus| object. However, by adding DeepLabv3-ResNet152~\cite{deeplabv3}, a heterogeneous model trained for solving semantic segmentation problem instead of object detection, our two-tier heterogeneous ensemble can detect the \verb|bus| object with high confidence of 0.573. 
{\it Third}, to gain a theoretical understanding of the robustness of our two-tier heterogeneity driven ensemble construction method, we provide a formal analysis of the heterogeneous ensemble robustness in terms of negative correlation, followed by extensive experiments. Our empirical results also validate the enhanced robustness of two-tier heterogeneous deep ensemble approach under both benign and adversarial settings.

\begin{figure*}[t!]
\centering
    \includegraphics[width=1.0\textwidth]{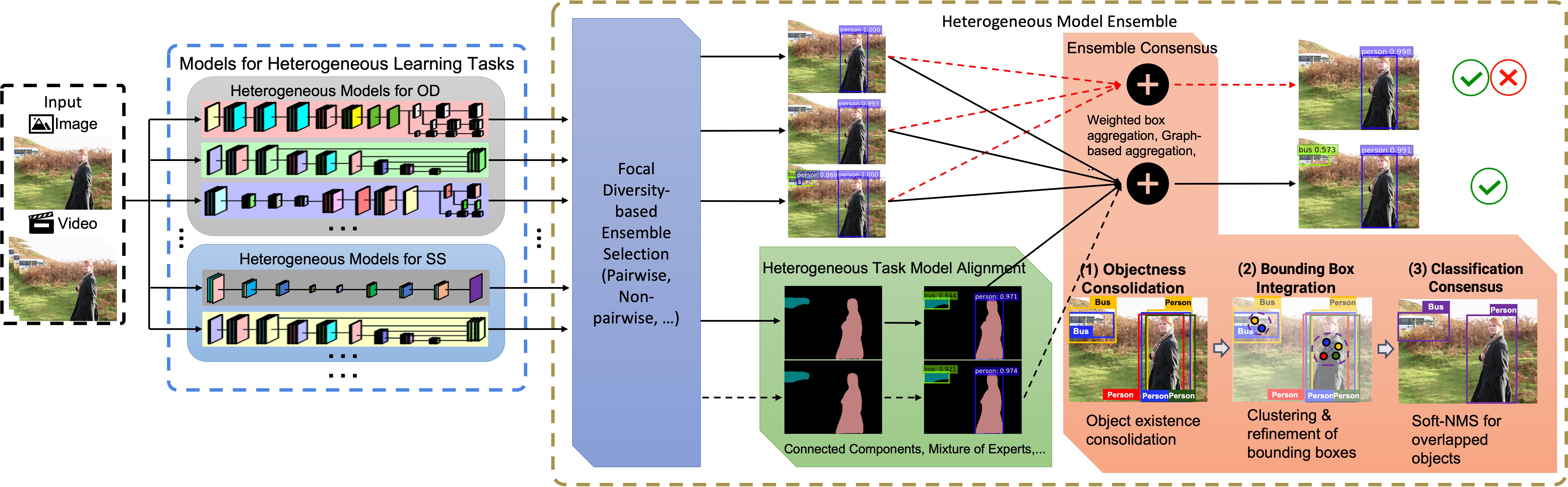}
    \caption{General Architecture for Heterogeneous Model Ensembles}
    \label{fig:het-model-ensemble-arch}
\end{figure*}

\section{Methodology}
The main idea of the proposed two-tier heterogeneous ensemble learning is to create each tier-1 ensemble from the heterogeneous models trained for solving the same learning problem, be it object detection or semantic segmentation, and then to create the next tier ensemble by combining the complementary wisdom from the heterogeneous ensemble teams trained for solving different types of learning problems, such as object detection and semantic segmentation. In this paper, we present the development of our general purpose architecture through a concrete prototype design and implementation, as shown in Figure~\ref{fig:het-model-ensemble-arch}. We use object detection models trained on a benchmark dataset (e.g., MS COCO~\cite{coco} or PASCAL VOC~\cite{voc}) to compose the tier-1 heterogeneous model ensemble, addressing the technical challenge of generating the ensemble prediction of all detected objects in terms of their bounding boxes and their classification. This requires a principled approach to multi-objective consensus building so that multiple detection prediction vectors from each member detectors can be effectively combined to improve the robustness of the tier-1 ensemble against both negative examples under benign scenarios and adversarial examples under malicious settings. In tier-2 heterogeneous ensemble teaming, we need to combine models trained for solving different learning problems. By combining the ensemble of heterogeneous models trained for object detection with one or more models trained for solving a different learning problem, such as semantic segmentation, we can further enhance the robustness of tier-1 heterogeneous ensembles. 

\begin{algorithm}[!h]
\caption{Ensemble Consensus for Object Detection}
\label{alg:od-ensemble}
\footnotesize
\begin{algorithmic}[1]
    \Procedure{DetectionEnsemble}{$\boldO_1, \boldO_2, ..., \boldO_N$, $\theta_{IoU}$}
    \State \textbf{Input}: $\boldO_1, \boldO_2, ..., \boldO_N$: the object detection results for an input image from $N$ different member models ($M_1, M_2,...,M_N$); $\theta_{IoU}$: the IoU threshold.
    \State \textbf{Output}: $\boldhatO$: the ensemble detection results.
    \State Initialize $\boldhatO = \{\}$.
    \For{each unique class label $c\in \{\boldO_1, \boldO_2, ..., \boldO_N\}$}
        \State Initialize $\boldV = \{\}$, $\boldE = \{\}$. \label{alg-od-ens:graph-start} \Comment{$\boldV$, $\boldE$ for the graph.}
        \LineComment{Obtain vertices for the graph.}
        \For{$i=1$ to $N$}
            \For{each object $\boldo \in \boldO_i$}
                \If{$\boldo_{[cls]} == c$}
                    \LineComment{Add each detected object as a vertex.}
                    \State $\boldV.add(i, \boldo)$
                \EndIf
            \EndFor
        \EndFor
        \LineComment{Build edges for the graph.}
        \For{each pair $\boldo, \boldo^{'} \in \binom{\boldV}{2}$}
            \State $w = IoU(\boldo_{[bbox]}, \boldo^{'}_{[bbox]})$
            \If{$\boldo != \boldo^{'}$ and $w \ge \theta_{IoU}$}
                \State $\boldE.add(\boldo, \boldo^{'}, w)$
            \EndIf
        \EndFor \label{alg-od-ens:graph-end}
        \LineComment{Clique Partitioning}
        \For{each $clique \in CliquePartition(\boldV, \boldE)$} \label{alg-od-ens:clique-partition}
            \LineComment{Reaching Consensus}
            \If{$clique.num\_models \ge N/2$}
                \State $\boldo^{\rmmax} = MaxConfidenceObject(clique)$ \label{alg-od-ens:find-max-conf}
                \LineComment{Weighted Bounding Box Aggregation}
                \State $\boldo^{\rmmax}_{[bbox]} = WeightedBoundingBox(clique)$ \label{alg-od-ens:weighted-bounding-box}
                \State $\boldhatO.append(\boldo^{\rmmax})$
                \LineComment{Address the problem of overlapped objects}
                \For{each object $\boldo \in clique - \boldo^{\rmmax}$} \label{alg-od-ens:soft-nms-start}
                    \State $\boldo_{[prob]}$=$\boldo_{[prob]}(1-IoU(\boldo^{\rmmax}_{[bbox]},\boldo_{[bbox]}))$
                    \State $\boldhatO.append(\boldo)$
                \EndFor \label{alg-od-ens:soft-nms-end}
            \EndIf
        \EndFor
    \EndFor
    \State \Return $\boldhatO$
    \EndProcedure
\end{algorithmic}
\end{algorithm}

\noindent {\bf Tier-1 Object Detection Ensemble: Model Teaming and Detection Integration.\/} We assume that there is a base model pool of diverse models trained for solving the same learning problem. We then leverage the focal diversity metrics (in the Appendix) to select a subset of highly diverse base models to form an ensemble team of high ensemble diversity. For example, we found that the models trained using heterogeneous DNN backbones, such as SSD512-VGG16, YOLOv3-DarkNet53 and FasterRCNN-ResNet50, are of high quality and yet with low failure dependency. For each tier-1 heterogeneous ensemble team, such as the ensemble of the three object detectors (SSD, YOLOv3, and FasterRCNN), we need to develop the integration consensus method that is customized to produce the ensemble detection output by effectively combining multiple prediction vectors for each detected object from each object detector. For example, a \verb|bus| object in Figure~\ref{fig:ensemble-seg-example-voc-intro}, the objectness vector, the bounding box vector, the classification vector produced by YOLOv3 (see the 2nd column in Figure~\ref{fig:ensemble-seg-example-voc-intro}) are quite different from those by FasterRCNN (see the 3rd column in Figure~\ref{fig:ensemble-seg-example-voc-intro}). How to build a robust consensus for object detection ensembles can be challenging.

In our prototype for object detection ensembles, we implement an effective three-step ensemble consensus method for combining multiple diverse object detectors.
Algorithm~\ref{alg:od-ensemble} provides a sketch of the pseudo code, which follows three steps.
{\it First}, we perform the objectness consolidation by examining those detected objects with their objectness confidence above a model-specific threshold and building a graph for each detected class label $c$. In the graph, each vertex represents a detected object of class label $c$ by a member object detector (e.g., $M_i$), and each weighted edge connects two detections with an IoU (Intersection over Union) score higher than the given IoU threshold $\theta_{IoU}$ ($=0.5$ by default), where this IoU score is used as the edge weight (Line~\ref{alg-od-ens:graph-start}$\sim$\ref{alg-od-ens:graph-end}).
{\it Second}, we perform the weighed bounding box ensemble by finding and aggregating the set of highly overlapped objects detected by different member models using the clique partitioning algorithm by maximizing the sum of the edge weights (Line~\ref{alg-od-ens:clique-partition}). For the resulting $clique$ with multiple vertices, it indicates that the corresponding objects from different member models have highly overlapped bounding boxes. We compute the final bounding box for each $clique$ using the weighted method for combining bounding boxes and integrating classification (Line~\ref{alg-od-ens:weighted-bounding-box}), where the confidence scores will serve as the default weights~\cite{softnms,robust-od-kdd,wbf}.
{\it Third}, we further address the problem of multiple overlapping objects of the same class label. For each clique reaching consensus, we examine every vertex in a clique and identify the vertex with the highest confidence, denoted as $\boldo^{\rmmax}$ (Line~\ref{alg-od-ens:find-max-conf}). Then for each remaining vertex $\boldo$ of the same class as $\boldo^{\rmmax}$, we recompute its classification confidence based on its bounding box overlapping (IoU) with that of $\boldo^{\rmmax}$, using the formula: $\boldo_{[prob]}=\boldo_{[prob]}(1-IoU(\boldo^{\rmmax}_{[bbox]},\boldo_{[bbox]}))$ (Line~\ref{alg-od-ens:soft-nms-start}$\sim$\ref{alg-od-ens:soft-nms-end}). By adding these objects, we can obtain the final ensemble detection results $\boldhatO$. Figure~\ref{fig:het-model-ensemble-arch} provides an architectural overview of our tier-1 heterogeneous ensemble consensus method by weighted algorithms for combining bounding boxes and integrating classification confidence for each detected object.

\begin{algorithm}[!h]
\caption{CCL-based Alignment}
\label{alg:ccl-alignment}
\footnotesize
\begin{algorithmic}[1]
    \Procedure{CCL-ALI}{$\bolds$, $minsize$}
    \State \textbf{Input}: $\bolds$: the output confidence scores of a semantic segmentation model; $minsize$: the minimum size of a connected component.
    \State \textbf{Output}: $\boldO^{*}$: all output objects including their class labels, confidence scores and bounding boxes.
    \State Initialize $\boldO^{*} = \{\}$.
    \LineComment{Extract the class label for each pixel; the shape of \bolds: $(C, H, W)$}
    \State $\p = \argmax(\bolds, axis=0)$
    \For{each unique class label $c\in \p$}
        \State $\boldt = \p.copy()$
        \State $\boldt[where(\p!=c)] = 0$ \label{alg-ccl-align:eliminate-other-class-labels}
        \Comment{Eliminate other class labels}
        \State ${\bf cc} = ConnectedComponentLabeling(\boldt)$ \label{alg-ccl-align:connected-component-labeling}
        \For{each connected component label $cl\in {\bf cc}$}
            \LineComment{Obtain the indexes of this component}
            \State ${\bf cc\_indices} = \argwhere({\bf cc} == cl)$
            \LineComment{Only consider the components above a threshold}
            \If{$size({\bf cc\_indices}) >= minsize$}
                \State $\boldo^{*}_{[cls]} = c$
                \LineComment{Obtain the bounding box of this component}
                \State $\boldo^{*}_{[bbox]} = BBox({\bf cc\_indices})$ \label{alg-ccl-align:bbox}
                \LineComment{Average the scores of this component}
                \State $\boldo^{*}_{[prob]} = Average(\bolds[c][{\bf cc\_indices}])$ \label{alg-ccl-align:average-confidence-scores}
                \State $\boldO^{*}.append((\boldo^{*}_{[cls]}, \boldo^{*}_{[prob]}, \boldo^{*}_{[bbox]}))$
            \EndIf
        \EndFor
    \EndFor
    \State \Return $\boldO^{*}$
    \EndProcedure
\end{algorithmic}
\end{algorithm}

\begin{figure*}[t!]
\centering
    \includegraphics[width=1.0\textwidth]{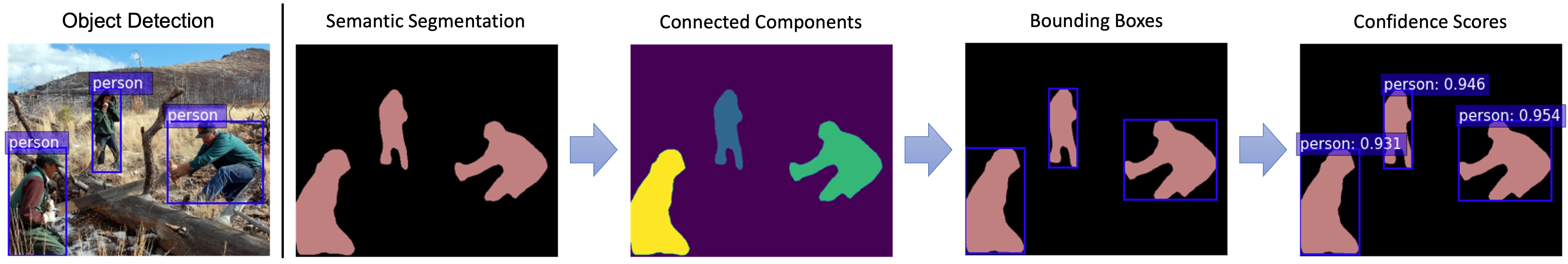}
    \caption{Align Semantic Segmentation with Object Detection (PASCAL VOC)}
    \label{fig:align-seg-to-det}
\end{figure*}

\noindent {\bf Tier-2 Heterogeneous Ensemble with Connected Component Labeling based Alignment.\/} We introduce a connected component labeling (CCL) based alignment algorithm to align and convert the outputs of semantic segmentation into the format used by the object detection models. Algorithm~\ref{alg:ccl-alignment} provides a sketch of the pseudo code.
It follows three main steps.
{\it First}, we only consider a specific class label $c$ each time and eliminate all others (not $c$) by setting them as 0 (background). Hence, we can perform connected component labeling to identify connected components (Line~\ref{alg-ccl-align:connected-component-labeling}) and assign each of them a component label $cl$.
{\it Second}, with an identified connected component, we then use a $BBox$ function to calculate the object bounding box by the min/max values of the coordinates in this component (Line~\ref{alg-ccl-align:bbox}).
{\it Third}, in order to determine a confidence score for this bounding box, we found that averaging the confidence scores for the class label $c$ within this component (Line~\ref{alg-ccl-align:average-confidence-scores}) can produce a very good confidence score for object detection.
Figure~\ref{fig:align-seg-to-det} (last three columns) illustrates this alignment process, and the confidence scores calculated by this method for all three \verb|person| detections are above 0.93, indicating high detection confidence, which is on par with well-trained object detectors.

\section{Formal Analysis}

We then formally analyze the robustness of heterogeneous deep ensembles using ensemble vulnerability. We show that the adversarial vulnerability of an ensemble of $N$ independent member models is proportional to $1/\sqrt{N}$, which demonstrates the increasing difficulty of attacking ensemble teams when the number of independent member models in an ensemble team increases. Moreover, low correlation among member models of an ensemble team will dramatically reduce such ensemble vulnerability. Therefore, the robustness of an ensemble team can be significantly improved by (1) increasing the number of independent member models and/or (2) reducing the correlation among member models for an ensemble team.

Assume that a DNN model $M_i$ is well-trained by minimizing a loss function $\ell_i(M_i(\x), \y)$ for the input sample $\x$ and ground-truth label $\y$. Even though the specific loss functions may vary for different DNN models and different learning tasks, such as object detection and semantic segmentation, they share the same objective to guide $M_i$ to produce accurate prediction results. Let $\mathcal{L}_i(\x, \y) = \ell_i(M_i(\x), \y)$. For simplicity, we can define the total loss for an ensemble team of $N$ member models as a weighted sum of all $N$ member model losses:
\begin{equation}
\small
\mathcal{L}_{ens}(\x, \y) = \sum_{i=1}^{N}{w_i \mathcal{L}_i(\x, \y)}
\label{eq:ens-loss}
\end{equation}
We assume all member models share an equal weight, that is $w_i=\frac{1}{N}$ for all $i=1,..., N$, such that $\sum_{i=1}^N w_i = 1$.

Under the adversarial setting, the goal of an adversary is to deceive the ensemble team to return erroneous prediction results imperceptibly. Hence, given an input sample $\x$, the objective function for generating an adversarial sample $\x_{adv}$ to perform adversarial attacks can be formulated as:
\begin{equation}
\small
\argmax_{\x_{adv}} \mathcal{L}_{ens}(\x_{adv}, \y) \quad
\text{s.t.} \quad ||\x_{adv}-\x||_p \leq \epsilon
\end{equation}
where feasible $\x_{adv}$ should lie in a $p$-norm bounded ball ($B(\x, \epsilon) = \{\x_{adv}, ||\x_{adv}-\x||_p<\epsilon\}$) with radius $\epsilon$ to be human-imperceptible.

\begin{definition}[Ensemble Vulnerability]
\label{def:ensemble-vulnerability}
Based on this observation, we have
\begin{equation}
\small
\Delta \mathcal{L}_{ens}(\x, \y, \epsilon) = \max_{||\delta||_p < \epsilon}{|\mathcal{L}_{ens}(\x+\delta,\y) - \mathcal{L}_{ens}(\x,\y)|}
\end{equation}
where $\Delta \mathcal{L}_{ens}$ captures the maximum absolute change of the ensemble loss value for a bounded change ($\delta$) on the input $\x$, such as adversarial perturbations. The vulnerability of an ensemble team can be measured by $\mathbb{E}_{\x} [\Delta \mathcal{L}_{ens}(\x, \y, \epsilon)]$, which is the expectation of $\Delta \mathcal{L}_{ens}$ on the input $\x$. 
\end{definition}
A robust ensemble is expected to have a small change in the loss ($\Delta \mathcal{L}_{ens}$) for a given input perturbation $\delta$. In practice, the input perturbation should be imperceptible, i.e., $\epsilon \rightarrow 0$. We assume that the loss functions are differentiable. Hence, $\Delta \mathcal{L}_{ens}$ can be approximated with a first-order Taylor expansion~\cite{first-order-taylor} as
\begin{equation}
\small
\begin{aligned}
\Delta \mathcal{L}_{ens}(\x, \y, \epsilon) &= \max_{||\delta||_p < \epsilon}{|\mathcal{L}_{ens}(\x+\delta,\y) - \mathcal{L}_{ens}(\x,\y)|} \\
&= \max_{||\delta||_p < \epsilon}{|\partial_{\x} \mathcal{L}_{ens}(\x, \y) \delta + O(\delta)|}
\end{aligned}
\label{eq:delta-L-expansion}
\end{equation}

We then formally show that ensembles of highly independent member models can significantly improve ensemble robustness under adversarial settings.

\begin{lemma}\label{lemma:vulnerability-adv}
Under a given adversarial attack, we can derive the ensemble vulnerability as
\begin{equation}
\small
\mathbb{E}_{\x} [\Delta \mathcal{L}_{ens}(\x, \y, \epsilon)] \propto \mathbb{E}_{\x}\left[||\partial_{\x} \mathcal{L}_{ens}(\x, \y)||_q \right]
\end{equation}
\end{lemma}

\begin{proof}[Proof of Lemma~\ref{lemma:vulnerability-adv}]
Following Equation~(\ref{eq:delta-L-expansion}), we have 
\begin{equation}
\small
\begin{aligned}
\Delta \mathcal{L}_{ens}(\x, \y, \epsilon) &= \max_{||\delta||_p < \epsilon}{|\partial_{\x} \mathcal{L}_{ens}(\x, \y) \delta + O(\delta)|} \\
&\approx \max_{||\delta||_p < \epsilon} |\partial_{\x} \mathcal{L}_{ens}(\x, \y) \delta |\\
&= ||\partial_{\x} \mathcal{L}_{ens}(\x, \y)||_q \cdot || \delta ||_p
\end{aligned}
\label{eq:delta-L-dual-norm}
\end{equation}
where $\frac{1}{p} + \frac{1}{q} = 1$ and $1 \leq p, q \leq \infty$ according to the definition of dual norm. Under a given adversarial attack, where $||\delta||_p $ is constant, we have 
\begin{equation}
\small
\begin{aligned}
\mathbb{E}_{\x}[\Delta \mathcal{L}_{ens}(\x, \y, \epsilon)] &\approx \mathbb{E}_{\x}\left[||\partial_{\x} \mathcal{L}_{ens}(\x, \y)||_q \cdot || \delta ||_p \right] \\
&\propto \mathbb{E}_{\x}\left[||\partial_{\x} \mathcal{L}_{ens}(\x, \y)||_q \right]
\end{aligned}
\end{equation}
\end{proof}

Let $\boldr_i$ denote the gradient of a member model $M_i$ on the input $\x$, i.e., $\boldr_i = \partial_{\x} \mathcal{L}_i(\x, \y)$. The correlation between the gradients of $M_i$ and $M_j$ is $\mathrm{Corr}(\boldr_i, \boldr_j)$. We have Theorem~\ref{theorem:vulnerability-adv-corr}.

\begin{theorem}[Ensemble Vulnerability under Adversarial Attacks]
\label{theorem:vulnerability-adv-corr}
If the gradients of member models of an ensemble team are i.i.d. with $\sigma^2 = \mathrm{Cov}(\boldr_i, \boldr_i)$ and zero mean (because a well-trained model is converged), we have 
\begin{equation}
\small
\begin{aligned}
\mathbb{E}_{\x}[\Delta \mathcal{L}_{ens}(\x, \y, \epsilon)]  &\propto \mathbb{E}_{\x}\left[\lVert \partial_{\x} \mathcal{L}_{ens}(\x, \y)\rVert_2 \right] \\
&= \sqrt{\frac{\sigma^2}{N}} \sqrt{1 + \frac{2}{N} \sum_{i=1}^N{\sum_{j=1}^{i-1}{\mathrm{Corr}(\boldr_i, \boldr_j)}}}
\end{aligned}
\label{eq:vulnerability-adv-corr-theorem}
\end{equation}
where the ensemble vulnerability $\mathbb{E}_{\x}[\Delta \mathcal{L}_{ens}(\x, \y, \epsilon)]$ is proportional to $\sqrt{\frac{\sigma^2}{N}} \sqrt{1 + \frac{2}{N} \sum_{i=1}^N{\sum_{j=1}^{i-1}{\mathrm{Corr}(\boldr_i, \boldr_j)}}}$ under a given adversarial attack.
\end{theorem}

\begin{proof}[Proof of Theorem~\ref{theorem:vulnerability-adv-corr}]
We follow Equation~(\ref{eq:ens-loss}) and have the joint gradient vector $\R$ for the ensemble team as
\begin{equation}
\small
\begin{aligned}
\R &= \partial_{\x} \mathcal{L}_{ens}(\x, \y) \\
&= \partial_{\x} \sum_{i=1}^{N}{w_i \mathcal{L}_i(\x, \y)} \\
&= \partial_{\x} \left(\frac{1}{N} \sum_{i=1}^{N}{\mathcal{L}_i(\x, \y)}\right) \\
&= \frac{1}{N} \sum_{i=1}^{N}{\partial_{\x}\mathcal{L}_i(\x, \y)} \\
&= \frac{1}{N} \sum_{i=1}^{N} \boldr_i 
\end{aligned}
\end{equation}
Then, the expectation of the square of the $L_2$-norm for this joint gradient $\R$ is
\begin{equation}
\small
\begin{aligned}
\mathbb{E}(\lVert \R \rVert_2^2) &= \mathbb{E}\left(\lVert \frac{1}{N} \sum_{i=1}^N{\boldr_i}\rVert_2^2\right) \\
&= \frac{1}{N^2}\mathbb{E}\left(\sum_{i=1}^N{\lVert\boldr_i\rVert_2^2} + 2 \sum_{i=1}^N{\sum_{j=1}^{i-1}{\boldr_i \boldr_j}}\right) \\
&= \frac{1}{N^2}\left(\sum_{i=1}^N\mathbb{E}({\lVert \boldr_i \rVert_2^2}) + 2 \sum_{i=1}^N{\sum_{j=1}^{i-1}\mathbb{E}({\boldr_i \boldr_j})}\right)
\end{aligned}
\label{eq:R-expansion-1}
\end{equation}
Given the gradients of all member models, including $M_i$ and $M_j$, are identically distributed, let $\sigma^2 = \mathrm{Cov}(\boldr_i, \boldr_i)$. Moreover, the model $M_i$ has converged with zero mean on its gradients, i.e., $\mathbb{E}(\boldr_i) = 0$. Hence, $\mathrm{Cov}(\boldr_i, \boldr_j) = \mathbb{E}(\boldr_i \boldr_j) - \mathbb{E}(\boldr_i)\mathbb{E}(\boldr_j) = \mathbb{E}(\boldr_i \boldr_j)$. Following the above Equation~(\ref{eq:R-expansion-1}), we have
\begin{equation}
\small
\begin{aligned}
\mathbb{E}(\lVert \R \rVert_2^2)
&= \frac{1}{N^2}\left(\sum_{i=1}^N\mathbb{E}({\lVert \boldr_i \rVert_2^2}) + 2 \sum_{i=1}^N{\sum_{j=1}^{i-1}\mathbb{E}({\boldr_i \boldr_j})}\right) \\
&= \frac{1}{N^2}\left(\sum_{i=1}^N\mathbb{E}(\boldr_i \boldr_i) + 2 \sum_{i=1}^N{\sum_{j=1}^{i-1}\mathbb{E}(\boldr_i \boldr_j)} \right)\\
&= \frac{1}{N^2}\left(\sum_{i=1}^N \mathrm{Cov}(\boldr_i, \boldr_i) + 2 \sum_{i=1}^N{\sum_{j=1}^{i-1}{\mathrm{Cov}(\boldr_i, \boldr_j)}}\right) \\
&= \frac{1}{N}\left(\sigma^2 + \frac{2}{N} \sum_{i=1}^N{\sum_{j=1}^{i-1}{\mathrm{Cov}(\boldr_i, \boldr_j)}}\right) \\
&= \frac{\sigma^2}{N}\left(1 + \frac{2}{N} \sum_{i=1}^N{\sum_{j=1}^{i-1}{\frac{\mathrm{Cov}(\boldr_i, \boldr_j)}{\sigma^2}}}\right) \\
&= \frac{\sigma^2}{N}\left(1 + \frac{2}{N} \sum_{i=1}^N{\sum_{j=1}^{i-1}{\mathrm{Corr}(\boldr_i, \boldr_j)}}\right)
\end{aligned}
\end{equation}
Hence, based on Lemma~\ref{lemma:vulnerability-adv}, $p=2$ and $q=2$, the ensemble vulnerability of this ensemble model is
\begin{equation}
\small
\begin{aligned}
\mathbb{E}_{\x}[\Delta \mathcal{L}_{ens}(\x, \y, \epsilon)]  &\propto \mathbb{E}_{\x}\left[\lVert \partial_{\x} \mathcal{L}_{ens}(\x, \y)\rVert_2 \right] \\
&= \sqrt{\frac{\sigma^2}{N}} \sqrt{1 + \frac{2}{N} \sum_{i=1}^N{\sum_{j=1}^{i-1}{\mathrm{Corr}(\boldr_i, \boldr_j)}}}
\end{aligned}
\label{eq:vulnerability-adv-corr}
\end{equation}
\end{proof}

The ensemble vulnerability for an ensemble team primarily depends on two factors: (1) the team size $N$ and (2) the member model correlations $\mathrm{Corr}(\boldr_i, \boldr_j)$. For an ideal case that all member models are independent of each other, we have $\mathrm{Corr}(\boldr_i, \boldr_j) = 0$ as follows.
\begin{equation}
\small
\begin{aligned}
\mathbb{E}_{\x}[\Delta \mathcal{L}_{ens}(\x, \y, \epsilon)] &\propto \mathbb{E}_{\x}\left[||\partial_{\x} \mathcal{L}_{ens}(\x, \y)||_2 \right] \\
&= \sqrt{ \frac{\sigma^2}{N}} \propto \frac{1}{\sqrt{N}}
\end{aligned}
\end{equation}
Therefore, the ensemble vulnerability is proportional to $1/\sqrt{N}$. It shows that the ensemble robustness can be significantly improved by increasing $N$, that is adding more independent member models. However, another case is when the correlations among member models are high, the ensemble may not improve the robustness. 
For an extreme case when all member models are perfect duplicates with $\mathrm{Corr}(\boldr_i, \boldr_j) = 1$, the following equation shows $\mathbb{E}_{\x}[\Delta \mathcal{L}_{ens}(\x, \y, \epsilon)] \propto \sigma$. 
\begin{equation}
\small
\begin{aligned}
\mathbb{E}_{\x}[\Delta \mathcal{L}_{ens}(\x, \y, \epsilon)] &\propto \mathbb{E}_{\x}\left[||\partial_{\x} \mathcal{L}_{ens}(\x, \y)||_2 \right] \\
&= \sqrt{\frac{\sigma^2}{N}} \sqrt{1 + \frac{2}{N}\frac{N(N-1)}{2}} \\
&= \sigma
\end{aligned}
\end{equation}
\begin{table}[h]
\centering
\caption{Base Models for MS COCO and PASCAL VOC}
\label{table:base-model-pools}
\scalebox{0.682}{
\small
\begin{tabular}{c|cc|cc}
\hline
Dataset & \multicolumn{2}{|c|}{MS COCO}           & \multicolumn{2}{c}{PASCAL VOC}          \\ \hline
ID      & Model                   & mAP (\%) & Model                 & mAP (\%) \\ \hline
0       & EfficientDetB5          & 49.3     & SSD512 (VGG16)        & 78.85    \\ \hline
1       & EfficientDetB6          & 50.1     & SSD512 (MobileNet)    & 75.53    \\ \hline
2       & EfficientDetB7          & 50.8     & \textbf{YOLOv3 (DarkNet53)}    & \textbf{81.48}    \\ \hline
3       & \textbf{DetectoRS (ResNeXt-101)} & \textbf{52.2}     & FasterRCNN (ResNet50) & 78.35    \\ \hline
4       & DetectoRS (ResNet50)    & 50.4     & CenterNet (ResNet18)  & 69.48    \\ \hline
5       & YOLOv5 (TTA)            & 48.5     & DeepLabv3 (ResNet152) & 56.49    \\ \hline
MAX     & DetectoRS (ResNeXt-101) & \textbf{52.2}     & YOLOv3 (DarkNet53)    & \textbf{81.48}    \\ \hline
AVG     &                         & 50.2     &                       & 73.36    \\ \hline
MIN     & EfficientDetB5          & 48.5     & DeepLabv3 (ResNet152) & 56.49   
\\ \hline
\end{tabular}
} 
\end{table}
\begin{table*}[h!]
\centering
\caption{Focal Diversity and mAP of Object Detector Ensembles on MS COCO, the higher focal diversity will select more robust ensembles (bold face) with higher mAP and smaller team size compared to the entire ensemble of all 6 base models. In comparison to the best single object detector (DetectoRS (ResNeXt-101) in Table~\ref{table:base-model-pools}) with 52.2\% mAP, all tier-1 heterogeneous ensembles of high diversity (bold face) significantly outperform this best single object detector by 1.2\%$\sim$1.7\% in mAP.}
\label{table:div-mAP-coco}
\small
\scalebox{1.0}{
\begin{tabular}{c|c|c|c|c|c|c|c|c|c|c}
\hline
\multicolumn{2}{c|}{Team Size}                  & 6           & \multicolumn{2}{c|}{5} & \multicolumn{2}{c|}{4} & \multicolumn{2}{c|}{3} & \multicolumn{2}{c}{2} \\ \hline
\multicolumn{2}{c|}{Ensemble Team}                                                        & 0,1,2,3,4,5 & \textbf{0,1,2,3,4} & 0,1,2,3,5 & \textbf{1,2,3,4} & 0,1,2,5 & \textbf{1,2,4} & 0,1,2 & \textbf{2,4}   & 1,2  \\ \hline
\multirow{2}{*}{\begin{tabular}[c]{@{}c@{}}Focal\\ Diversity\end{tabular}} & Pairwise    & 0.149       & \textbf{0.557}     & 0.092     & \textbf{0.774}   & 0       & \textbf{0.574} & 0     & \textbf{0.907} & 0    \\ \cline{2-11}
                                                                           & Non-pairwise & 0.088       & \textbf{0.518}     & 0.079     & \textbf{0.737}   & 0       & \textbf{0.497} & 0     & \textbf{0.888} & 0    \\ \hline
\multicolumn{2}{c|}{\multirow{2}{*}{mAP (\%)}}                                                               & \multirow{2}{*}{53.2}        & \textbf{53.9}      & 52.4      & \textbf{53.7}    & 50.7    & \textbf{53.6}  & 51.5  & \textbf{53.4}  & 51.4 \\
\multicolumn{2}{c|}{}  & & \tabfiguresmall{width=0.02\textwidth}{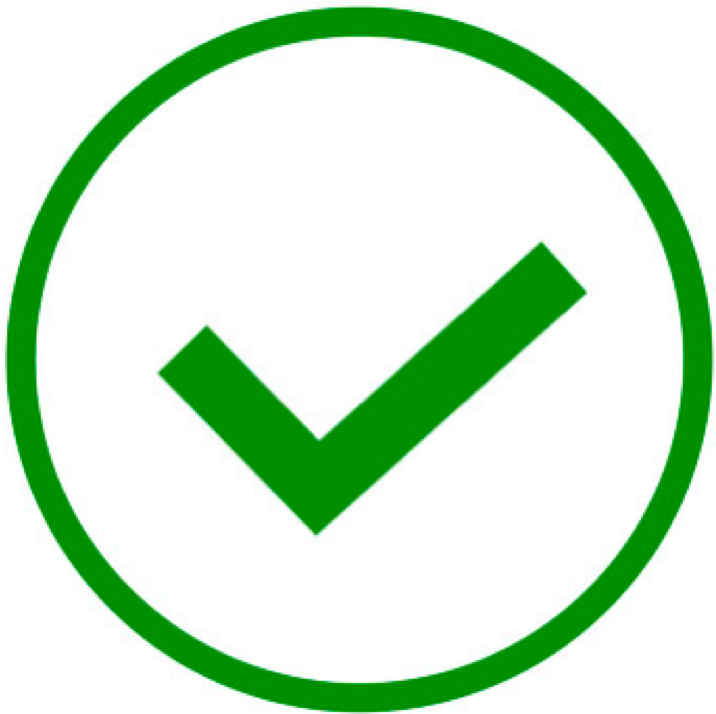} & \tabfiguresmall{width=0.02\textwidth}{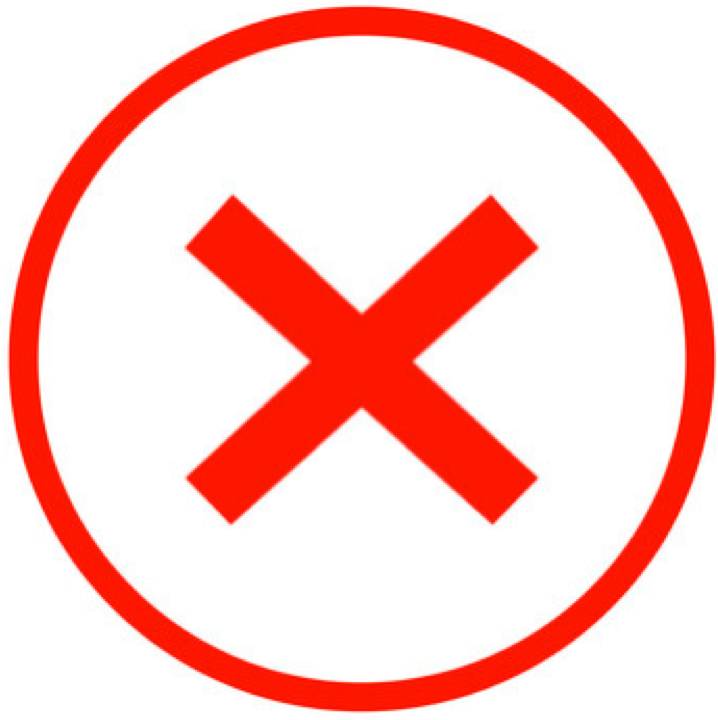} & \tabfiguresmall{width=0.02\textwidth}{c} & \tabfiguresmall{width=0.02\textwidth}{w} & \tabfiguresmall{width=0.02\textwidth}{c} & \tabfiguresmall{width=0.02\textwidth}{w} & \tabfiguresmall{width=0.02\textwidth}{c} & \tabfiguresmall{width=0.02\textwidth}{w}\\ \hline
\end{tabular}
} 
\end{table*}
\begin{figure*}[h!]
\centering
    \includegraphics[width=1.0\textwidth]{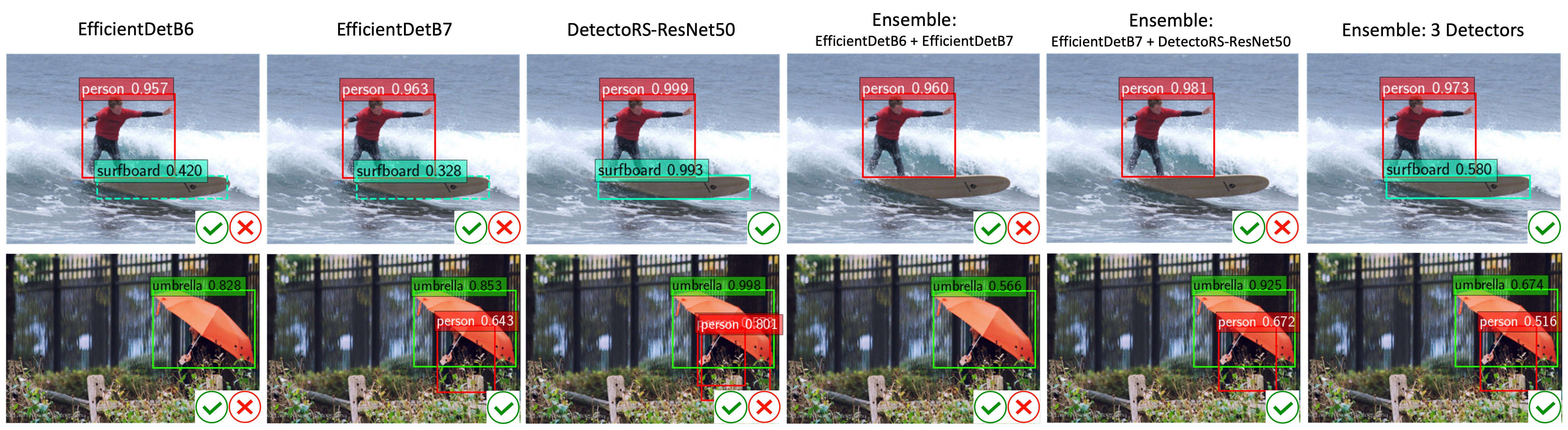}
    \caption{Object Detection Examples on MS COCO, visual examples illustrate that the ensemble of 3 detectors with high focal diversity has higher mAP compared to any of the sub-ensembles of 2 detectors.}
    \label{fig:ensemble-examples-coco-exp}
\end{figure*}
It indicates that the ensemble will have the same robustness as a single model regardless of how many member models it has.
We follow~\cite{EnsembleBenchCVPR,EnsembleBenchICDM} and present novel focal diversity metrics for object detection in the Appendix to capture such negative correlation among member detectors of an ensemble team, including pairwise and non-pairwise focal diversity metrics.
Our heterogeneous ensembles leverage two-tier heterogeneity in model learning to significantly reduce the member model correlations and thus improve ensemble robustness.

\section{Experimental Analysis}\label{sec:exp-analysis}

We evaluate the robustness of the proposed two-tier heterogeneous ensemble approach on benchmark datasets: MS COCO~\cite{coco} and PASCAL VOC~\cite{voc}, using Intel i7-10700K with the NVIDIA GeForce RTX 3090 (24GB) GPU, installed with Ubuntu 20.04 and CUDA 11.
Table~\ref{table:base-model-pools} lists all base models for MS COCO and PASCAL VOC. The six heterogeneous object detectors for MS COCO include EfficientDetB5-B7~\cite{efficientdet}, two DetectoRS models with different backbones~\cite{detectors}, and YOLOv5 with test time augmentation (TTA)~\cite{yolov5,wbf}. For PASCAL VOC, we use five heterogeneous object detectors: SSD512 (VGG16, MobileNet)~\cite{ssd}, YOLOv3 (DarkNet53)~\cite{yolov3}, FasterRCNN (ResNet50)~\cite{faster-rcnn} and CenterNet (ResNet18)~\cite{centernet}, and one semantic segmentation (SS) model, DeepLabv3~\cite{deeplabv3}, which is converted into object detections (OD) using our connected component based alignment method.

\noindent {\bf Diversity and Robustness of Tier-1 Heterogeneous Ensemble.\/} We first measure the focal diversity scores and mean average precision (mAP) for the tier-1 ensembles of different sizes on MS COCO. The six heterogeneous object detectors can form ensembles of size $N$ ($2 \le N \le 6$). For example, the COCO ensemble made up by all six detectors is \verb|0,1,2,3,4,5|. Multiple sub-ensembles of size $N$ can be composed from the six base models ($2 \le N \le 5$). We can rank the sub-ensembles of size $N$ by their focal diversity scores. Table~\ref{table:div-mAP-coco} shows two extreme examples for each given size $N$: one is the ensemble of high diversity in bold, which will be selected, and another is the ensemble of low diversity, which will be pruned out. We make two observations.
{\it First,} ensembles with high focal diversity also have high mAP. Compared to the low diversity ensembles, the high diversity ensembles can effectively improve the mAP by 1.5\%$\sim$2.1\%. For example, the high diversity ensemble team \verb|1,2,4| achieved the mAP of 53.6\%, improving the mAP of the low diversity ensemble \verb|0,1,2| by 2.1\%. 
{\it Second,} all ensembles of high diversity outperform the ensemble made up of all six detectors for COCO with 53.2\% mAP, outperforming the best base detection model DetectoRS (ResNeXt-101) of 52.2\% mAP by 1.2\%$\sim$1.7\%.

\begin{table*}[h]
\centering
\caption{Focal Diversity and mAP of Object Detector Ensembles on PASCAL VOC. Bold highlighted sub-ensemble teams have high focal diversity scores (third and fourth rows), resulting in higher ensemble mAP (last row). The top-2 performing sub-ensembles are the two-tier heterogeneous ensembles: \protect\texttt{0,1,2,3,5} and \protect\texttt{0,2,3,5}, both include the semantic segmentation model \protect\texttt{5} (DeepLabv3-ResNet152).}
\label{table:div-mAP-voc}
\small
\scalebox{1.0}{
\begin{tabular}{c|c|c|c|c|c|c|c|c|c|c}
\hline
\multicolumn{2}{c|}{Team Size}                                                            & 6           & \multicolumn{2}{c|}{5} & \multicolumn{2}{c|}{4} & \multicolumn{2}{c|}{3} & \multicolumn{2}{c}{2} \\ \hline
\multicolumn{2}{c|}{Ensemble Team}                                                        & 0,1,2,3,4,5 & \textbf{0,1,2,3,5} & 0,1,2,4,5 & \textbf{0,2,3,5}   & 0,1,4,5   & \textbf{0,2,3}     & 1,3,4     & \textbf{2,3}       & 3,4       \\ \hline
\multirow{2}{*}{\begin{tabular}[c]{@{}c@{}}Focal\\ Diversity\end{tabular}} & Pairwise    & 0.247       & \textbf{0.997}     & 0         & \textbf{1}         & 0.024     & \textbf{0.991}     & 0.277         & \textbf{1}         & 0         \\ \cline{2-11}
                                                                           & Non-pairwise & 0.156       & \textbf{1}         & 0         & \textbf{1}         & 0.019     & \textbf{0.995}     & 0.227         & \textbf{1}         & 0         \\ \hline
\multicolumn{2}{c|}{\multirow{2}{*}{mAP (\%)}}                                                                  & \multirow{2}{*}{84.13}       & \textbf{84.43}     & 83.55     & \textbf{84.25}     & 81.70     & \textbf{84.15}     & 80.31     & \textbf{83.26}     & 77.22 \\ 
\multicolumn{2}{c|}{}  & & \tabfiguresmall{width=0.02\textwidth}{c} & \tabfiguresmall{width=0.02\textwidth}{w} & \tabfiguresmall{width=0.02\textwidth}{c} & \tabfiguresmall{width=0.02\textwidth}{w} & \tabfiguresmall{width=0.02\textwidth}{c} & \tabfiguresmall{width=0.02\textwidth}{w} & \tabfiguresmall{width=0.02\textwidth}{c} & \tabfiguresmall{width=0.02\textwidth}{w}\\ \hline
\end{tabular}
} 
\end{table*}

\begin{figure*}[h!]
\centering
    \includegraphics[width=1.0\textwidth]{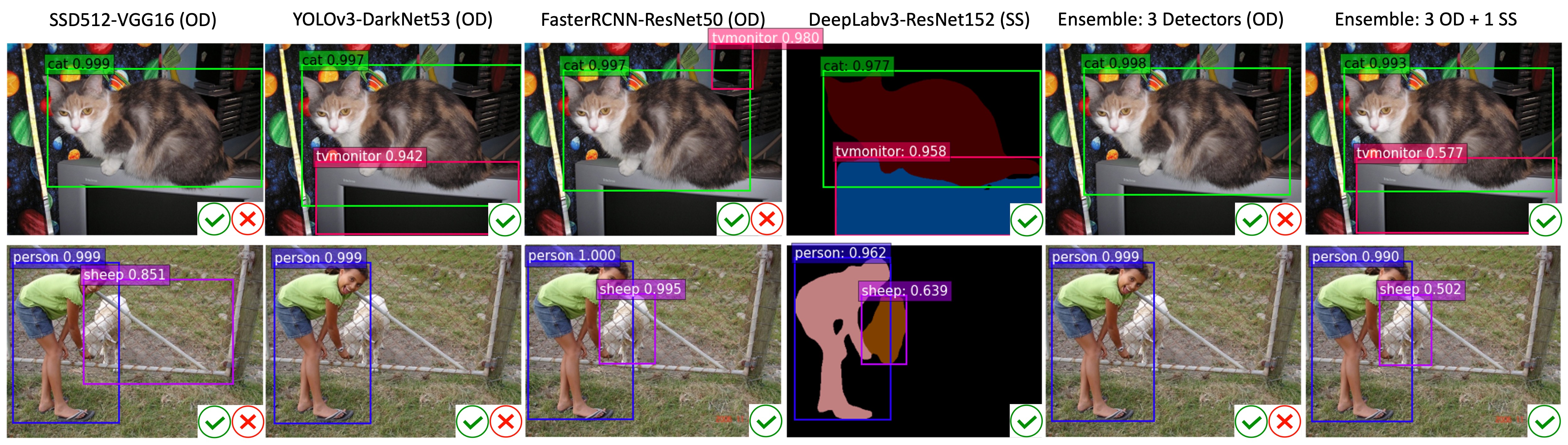}
    \caption{Object Detection and Segmentation Examples on PASCAL VOC, visual examples illustrate scenarios of heterogeneous ensembles (6th column: 3OD + 1SS) outperforming homogeneous ensembles (5th column: 3 Detectors).}
    \label{fig:ensemble-seg-examples-voc-exp}
\end{figure*}

Figure~\ref{fig:ensemble-examples-coco-exp} provides two visualization examples on COCO by comparing the three model ensemble \verb|1,2,4| (i.e., EfficientDetB6, EfficientDetB7 and DetectoRS-ResNet50 in the last column) with the ensembles of two detectors \verb|1,2| (i.e., EfficientDetB6 and EfficientDetB7) and \verb|2,4| (i.e., EfficientDetB7 and DetectoRS-ResNet50) in the fourth and fifth columns respectively. It is observed that the high diversity ensemble \verb|1,2,4| can make correct detections even though two out of three member object detectors make partial prediction errors. This is because the ensembles of high focal diversity can effectively repair the mistakes and output the correct object detection results through our confidence weighted bounding box ensemble consensus method, showing the improved robustness under both benign scenarios and under attacks. Furthermore, heterogeneous object detectors tend to make different errors on negative samples. Consider the first image in Figure~\ref{fig:ensemble-examples-coco-exp}, EfficientDetB6 and EfficientDetB7 detected the \verb|surfboard| with the confidence scores of 0.420 and 0.328 (in the dashed bounding box) respectively, which is lower than the display threshold of 0.5, while DetectoRS-ResNet50 successfully detected the \verb|surfboard| with a high confidence score of 0.993. Our weighted bounding box ensemble consensus method can effectively aggregate these results and detect the \verb|surfboard| object with a higher confidence score of 0.580 than the display threshold.
In contrast, the member models in the low diversity ensemble \verb|1,2| tend to make similar detections (see the first row) and often fail to repair each other's prediction errors unlike the high diversity ensemble \verb|2,4|, which explains the reduced mAP for these low diversity ensembles.

\noindent\textbf{Diversity and Robustness of Two-tier Heterogeneous Ensemble.} We next evaluate the diversity and robustness of our proposed heterogeneous ensembles on PASCAL VOC. We introduce another tier of heterogeneity by adding model \verb|5| DeepLabv3-ResNet152, a pre-trained semantic segmentation model, into the base model pool. Table~\ref{table:div-mAP-voc} reports the focal diversity scores and mAP for the entire ensemble of six models, high diversity ensembles and low diversity ensembles of size $N$ ($2 \le N \le 5$). We highlight two observations.
{\it First,} all high diversity ensembles significantly outperform the low diversity ensembles in mAP by 0.88\%$\sim$6.04\% and improve the best base model YOLOv3-DarkNet53 of 81.48\% mAP by 1.78\%$\sim$2.95\%.
{\it Second,} the high diversity two-tier heterogeneous ensembles (i.e., \verb|0,1,2,3,5| and \verb|0,2,3,5|, including the semantic segmentation model \verb|5|, DeepLabv3-ResNet152) can further improve the high diversity tier-1 ensembles (e.g., \verb|0,2,3|) in mAP. They offer 2.95\% and 2.77\% increases in mAP respectively over the best single detector YOLOv3 (DarkNet53) of 81.48\% mAP.

\begin{table*}
\centering
\caption{Ensemble Performance (mAP (\%)) under Adversarial Attacks on PASCAL VOC. Both TOG-vanishing and TOG-mislabeling attacks are configured with four settings of adversarial perturbation $\epsilon$.
}
\label{table:det-set-voc-adv}
\small
\scalebox{1.0}{
\begin{tabular}{c|c|c|c|c|c|c|c|c|c|c}
\hline
\multicolumn{2}{c|}{\multirow{2}{*}{Model}} &
  \multirow{2}{*}{Benign} &
  \multicolumn{4}{c|}{TOG-v Attacks (default: $\epsilon$=8/255)} &
  \multicolumn{4}{c}{TOG-m Attacks (default: $\epsilon$=8/255)} \\ \cline{4-11}
\multicolumn{2}{c|}{}                       &         & 4/255   & 8/255   & 16/255  & 32/255  & 4/255   & 8/255    & 16/255   & 32/255    \\ \hline
\multicolumn{2}{c|}{Victim: SSD512 (VGG16)} & 73.11   & \HIGHLIGHT{44.93}   & \HIGHLIGHT{31.23}   & \HIGHLIGHT{25.12}   & \HIGHLIGHT{19.62}   & \HIGHLIGHT{8.40}     & \HIGHLIGHT{1.20}      & \HIGHLIGHT{0.79}     & \HIGHLIGHT{0.66}      \\ \hline
\multirow{6}{*}{Ensemble} &
  \multirow{2}{*}{3 OD} &
  85.98 &
  76.61 &
  67.55 &
  61.22 &
  57.13 &
  76.37 &
  65.93 &
  62.2 &
  58.74 \\
       &                                   & (1.18$\times$) & (1.71$\times$) & (2.16$\times$) & (2.44$\times$) & (2.91$\times$) & (9.09$\times$) & (54.94$\times$) & (78.73$\times$) & (89$\times$)     \\ \cline{2-11}
       & \multirow{2}{*}{4 OD}             & 84.68   & 77.15   & 69.46   & 63.29   & 57.85   & 77.39   & 68.36    & 64.16    & 62.39     \\
       &                                   & (1.16$\times$) & (1.72$\times$) & (2.22$\times$) & (2.52$\times$) & (2.95$\times$) & (9.21$\times$) & (56.97$\times$) & (81.22$\times$) & (94.53$\times$)  \\ \cline{2-11}
       & \multirow{2}{*}{3 OD + 1 SS}      & 86.22   & \textbf{78.77}   & \textbf{70.79}   & \textbf{64.42}   & \textbf{62.43}   & \textbf{79.71}   & \textbf{71.75}    & \textbf{68.26}    & \textbf{66.65}     \\
       &                                   & (1.18$\times$) & (\textbf{1.75$\times$}) & (\textbf{2.27$\times$}) & (\textbf{2.56$\times$}) & (\textbf{3.18$\times$}) & (\textbf{9.49$\times$}) & (\textbf{59.79$\times$}) & (\textbf{86.41$\times$}) & (\textbf{100.98$\times$}) \\ \hline
\end{tabular}
}
\end{table*}

\begin{figure*}[h!]
\centering
    \includegraphics[width=1.0\textwidth]{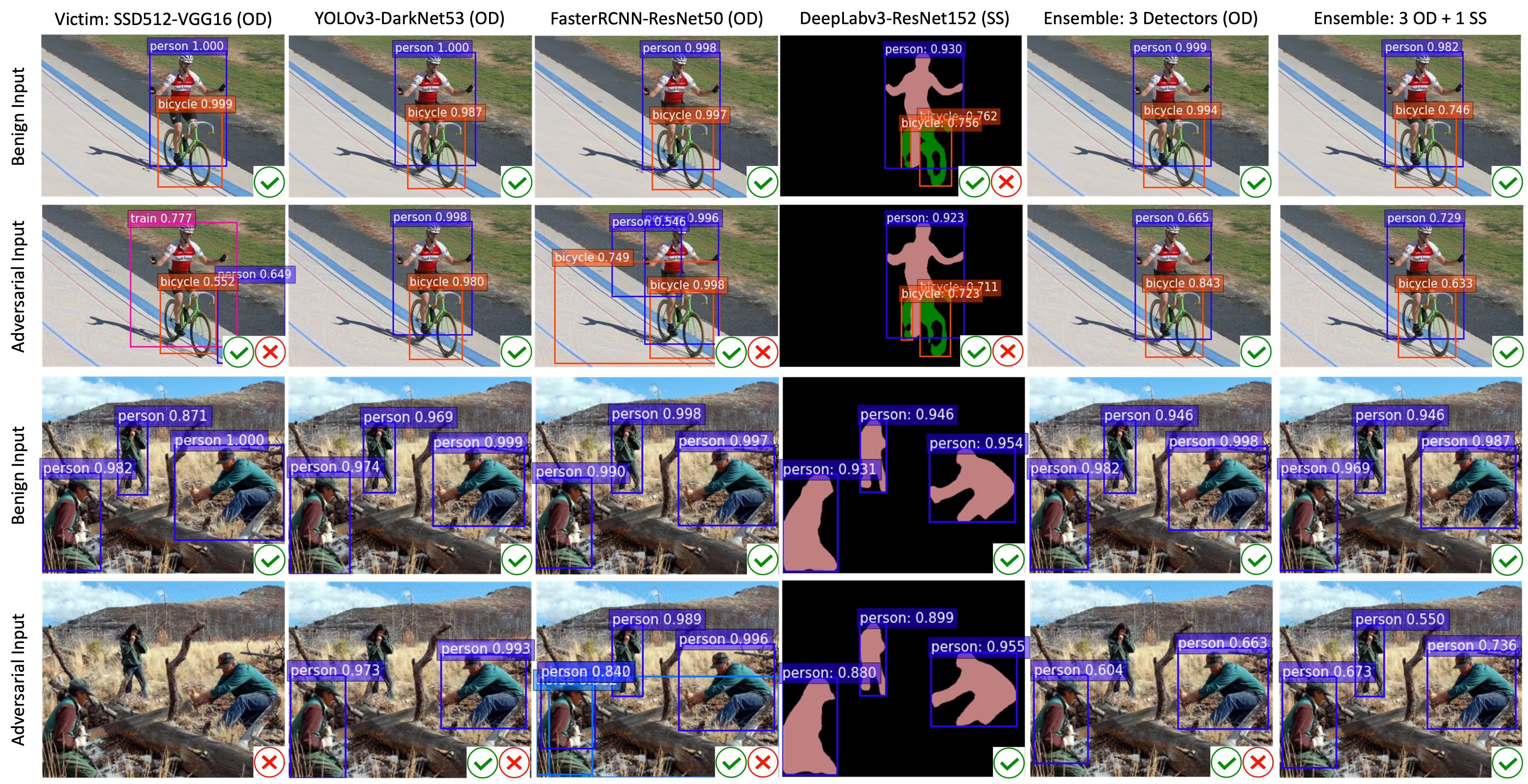}
    \caption{Adversarial Attack Examples on PASCAL VOC (TOG-m: top two images, TOG-v: bottom two images, $\epsilon$ = 8/255)}
    \label{fig:adv-examples-voc-exp}
\end{figure*}

Figure~\ref{fig:ensemble-seg-examples-voc-exp} provides two visualization examples for the high diversity ensembles \verb|0,2,3| (3 OD) and \verb|0,2,3,5| (3 OD + 1 SS). It shows that even though the tier-1 heterogeneous ensemble (\verb|0,2,3|) of three diverse object detectors fails to make correct predictions, by incorporating the semantic segmentation model \verb|5| DeepLabv3-ResNet152, this two-tier heterogeneous ensemble can produce correct object detection results, showing higher robustness on test examples that are negative examples for two or more member models.
For example, both SSD512-VGG16 and FasterRCNN-ResNet50 missed the same object: \verb|tvmonitor| (first row), and the tier-1 ensemble of three object detectors also failed to identify the \verb|tvmonitor|, indicating that this example might be hard to most of the object detectors. By adding DeepLabv3-ResNet152, a model which is trained on the same dataset VOC but for solving a different learning problem, the semantic segmentation, we show that DeepLabv3-ResNet152 can clearly recognize the \verb|tvmonitor| with high confidence. Our two-tier heterogeneous ensemble will combine the detection results derived from three object detectors and one semantic segmentation model, and can make the correct object detections. Table~\ref{table:div-mAP-voc} also shows that the two-tier ensemble (\verb|0,2,3,5|) enhanced by DeepLabv3-ResNet152 can outperform the tier-1 ensemble (\verb|0,2,3|) and offer a slight increase in the overall mAP, even though DeepLabv3-ResNet152 as an individual model has a low mAP of 56.49\%, compared to the three object detectors. This further shows that two-tier heterogeneous ensembles can increase the ensemble diversity of ensemble member models and improve the overall ensemble robustness.

\begin{table*}[!h]
\centering
\caption{Victims: FasterRCNN (ResNet50) $\|$ YOLOv3 (DarkNet53)}
\label{table:adv-faster-rcnn-yolov3-voc}
\small
\scalebox{1.0}{
\begin{tabular}{l|c|c|c}
\hline
\multirow{2}{*}{Model}        & \multirow{2}{*}{Benign mAP (\%)} & \multicolumn{2}{c}{Under attack mAP (\%)} \\ \cline{3-4}
                              &                                  & TOG-m               & TOG-v               \\ \hline
Victim: FasterRCNN (ResNet50) & 80.06                            & 2.14                & 0.14                \\ \hline
\hfill Ensemble Team                    & 85.98 (\textbf{1.07$\times$})                    & 82.29 (\textbf{38.45$\times$})      & 81.23 (\textbf{\textbf{580.21$\times$}})     \\ \hline \hline
Victim: YOLOv3 (DarkNet53)    & 82.98                            & 3.15                & 1.24                \\ \hline
\hfill Ensemble Team                     & 85.98 (\textbf{1.04$\times$})                    & 76.15 (\textbf{24.17$\times$})      & 75.88 (\textbf{61.19$\times$})     \\ \hline 
\end{tabular}}
\end{table*}

\noindent \textbf{Adversarial Robustness of Two-Tier Heterogeneous Ensemble.}
Adversarial attacks often target at misleading one victim (or target) model in practice~\cite{adv-examples-seg-det,robust-adv-perturb-proposal-models,transferable-adv-attack-od,od-attack-tps}. We assume that an adversary only has access to one object detector and construct adversarial examples to attack this victim model. In this set of experiments, we evaluate the adversarial robustness of our proposed two-tier heterogeneous ensembles against such adversarial attacks.
We choose two different TOG~\cite{od-adv-lens-esorics,od-attack-tps} attack algorithms: object mislabeling attacks (TOG-m) and object vanishing attacks (TOG-v), which use a hyperparameter $\epsilon$ to control the maximum perturbation value on image pixels. A larger $\epsilon$ indicates a stronger TOG attack. The default $\epsilon$ is 8/255. We randomly subsample 500 test examples from PASCAL VOC to generate adversarial examples under different $\epsilon$ to attack a victim detector, say SSD512-VGG16, and evaluate the mAP under both benign and adversarial settings for three different ensemble teams: the baseline tier-1 ensemble (3 OD: \verb|0,2,3|), which consists of three object detectors, another tier-1 ensemble (4 OD: \verb|0,2,3,4|), which incorporates another object detector, CenterNet-ResNet18, and the two-tier heterogeneous ensemble (3 OD + 1 SS: \verb|0,2,3,5|), which adds a pre-trained VOC semantic segmentation model (DeepLabv3-ResNet152) to this baseline tier-1 ensemble team (\verb|0,2,3|). Table~\ref{table:det-set-voc-adv} reports the results. Figure~\ref{fig:adv-examples-voc-exp} shows one visualization example under the TOG-v attack. We highlight three interesting observations.
{\it First,} adversarial attacks (TOG-v and TOG-m) can cause a drastic drop in the mAP for the single well-trained object detector, SSD512-VGG16, from the benign mAP of 73.11\% to 19.62\%$\sim$44.93\% under TOG-v and 0.66\%$\sim$8.40\% under TOG-m.
{\it Second,} the transferability of the attack to both DeepLabv3-ResNet52 and the other two object detectors: YOLOv3 or FasterRCNN is not high. Consider the first two rows in Figure~\ref{fig:adv-examples-voc-exp}, the TOG-m attack causes SSD512 to misclassify the \verb|person| object as \verb|train| with high confidence, and causes a transfer attack to FasterRCNN by fabricating additional \verb|person| and \verb|bicycle| objects (see the second row and third column), while the attack fails to transfer to YOLOv3 object detection model and DeepLabv3-ResNet152 semantic segmentation model. As a result, both heterogeneous ensemble teams are attack-resilient. For the second example in the last two rows in Figure~\ref{fig:adv-examples-voc-exp}, the TOG-v attack against SSD512 object detector causes all three \verb|person| objects to vanish. This attack is partially transferred to YOLOv3 and hence the tier-1 ensemble, causing one \verb|person| object to vanish in the YOLOv3 detection and the tier-1 ensemble detection. However, the TOG-v attack fails to transfer to FasterRCNN and DeepLabv3-ResNet152. Hence, the two-tier heterogeneous ensemble remains to be attack-resilient.
{\it Third,} the two-tier heterogeneous ensemble can significantly improve the mAP under both benign (no attack) and TOG attack scenarios. It improves the benign mAP over the victim model by 13.11\%. Under both TOG-m and TOG-v attacks, this two-tier heterogeneous ensemble consistently achieves the highest mAP under different attack levels with different $\epsilon$ and outperforms the victim model by 1.75$\sim$100.98$\times$ and the other tier-1 ensemble (4 OD: \verb|0,2,3,4|) by 1.13\%$\sim$4.58\% in mAP. With the default $\epsilon$ = 8/255, this two-tier heterogeneous ensemble shows high robustness with the mAP of 71.75\% and 70.79\% respectively, which is on par with the benign mAP of 73.11\% for the victim object detector SSD512-VGG16.

\noindent \textbf{Different Object Detectors under Attacks.}
We found similar observations when different object detectors are the attack victims. Table~\ref{table:adv-faster-rcnn-yolov3-voc} shows when FasterRCNN (ResNet50) or YOLOv3 (DarkNet53) is the victim, our proposed heterogeneous ensemble is equally robust. This set of experiments further demonstrates that our heterogeneous ensemble approach can effectively protect different types of object detectors and boost ensemble robustness under both benign and adversarial settings.

\begin{table}[h!]
\centering
\caption{Our Heterogeneous Ensemble Approach vs. NMS}
\label{table:ode-mAP-evaluation-nms}
\small
\scalebox{0.91}{
\begin{tabular}{c|c|c|c|c|c}
\hline
\multicolumn{2}{c|}{Ensemble Team}                                                       & 0,2,3,5 & 0,2,3 & 2,3 & 0,2 \\ \hline
\multirow{2}{*}{\begin{tabular}[c]{@{}c@{}}Benign\\ mAP (\%)\end{tabular}}       & Ours & \textbf{86.22}                                                                         & \textbf{85.98}                                                             & \textbf{85.32}                                                                      & \textbf{84.22}       \\ \cline{2-6}
                                                                                 & NMS  & N/A                                                                           & 80.95                                                             & 80.17                                                                      & 79.54       \\ \hline
\multirow{2}{*}{\begin{tabular}[c]{@{}c@{}}Under attack\\ mAP (\%)\end{tabular}} & Ours & \textbf{71.75}                                                                         & \textbf{65.93}                                                             & \multirow{2}{*}{\begin{tabular}[c]{@{}c@{}}No victim\\ (SSD)\end{tabular}} & \textbf{40.00}       \\ \cline{2-4} \cline{6-6}
                                                                                 & NMS  & N/A                                                                           & 44.15                                                             &                                                                            & 30.46      \\ \hline
\multicolumn{6}{l}{\begin{tabular}[l]{@{}l@{}}
Model 2: 81.48\% mAP (highest member); N/A: NMS cannot support\\DeepLabv3, a semantic segmentation model in ensemble consensus.\end{tabular}}
\end{tabular}
}
\end{table}

\noindent {\bf Ablation Study.} Table~\ref{table:ode-mAP-evaluation-nms} shows the comparison of 4 ensembles consisting of 4 heterogeneous models 0, 2, 3, 5 on VOC. We highlight three interesting observations.
{\it First,} our approach consistently outperforms the baseline NMS approach for performing ensemble consensus under both benign and adversarial settings and significantly improves the mAP by 4.68\%$\sim$21.78\%.
{\it Second,} our approach provides the support for two-tier heterogeneous ensembles. The two-tier heterogeneous ensemble, \verb|0,2,3,5|, incorporates a semantic segmentation model (DeepLabv3, model \verb|5|) into this tier-1 ensemble of 3 object detectors (\verb|0,2,3|), significantly strengthens the ensemble robustness under adversarial attacks and improves the under attack mAP by 5.82\%.
{\it Third,} the ensemble \verb|0,2| only achieved the focal diversity scores of 0.519 (pairwise) and 0.457 (non-pairwise), which is much lower than the other two ensembles under attack (\verb|0,2,3| and \verb|0,2,3,5| in Table~\ref{table:div-mAP-voc}) with over 0.991 focal diversity scores. The lower ensemble diversity of this ensemble of SSD512 and YOLOv3 also explains its larger drop in mAP (84.22\%$\rightarrow$40\%) under attack in Table~\ref{table:ode-mAP-evaluation-nms}.
Therefore, the ensembles with a higher diversity score indeed provide higher mAP performance under both benign and attack scenarios.

\section{Conclusion}
We presented a novel DNN ensemble learning approach by exploring model learning heterogeneity and focal diversity of ensemble teams.
{\it First,} we show that combining heterogeneous DNN models trained on the same learning problem, e.g., object detection, with high focal diversity can improve the generalization performance in detection mAP in situations where there exists disagreement among individual member models.
{\it Second,} we introduce a connected component based alignment method to further boost model learning heterogeneity by incorporating models trained for solving different learning problems, such as two-tier heterogeneous ensembles of semantic segmentation models with object detection models.
{\it Third,} we present a formal analysis of heterogeneous ensemble robustness in terms of negative correlation and comprehensive experimental evaluation. We show that combining focal diversity and model learning heterogeneity can effectively boost ensemble robustness, especially when individual member models fail on negative examples or are vulnerable to adversarial attacks.


\section*{Acknowledgment}
This research is partially sponsored by the NSF CISE grants 2302720, 2312758, 2038029, an IBM faculty award, and a grant from CISCO Edge AI program.




\bibliographystyle{IEEEtran}
\bibliography{IEEEabrv,reference}
%



\appendix

Although the theoretical analysis has shown us that diverse member models with low correlation can improve the ensemble robustness against adversarial attacks, it does not offer practical guidelines for measuring and enforcing the diversity of an ensemble team.
Existing ensemble diversity metrics~\cite{EnsembleBenchCogMI,EnsembleBenchCVPR,EnsembleBenchICDM,diversityaccuracy,generalizeddiversity,binarydisagreement} are proposed to measure the ensemble diversity for single task learners, such as image classifiers. For multi-task learners, such as object detectors, it is challenging to measure the ensemble diversity due to their complex outputs, such as the bounding boxes, class labels, and classification confidence scores for object detection.
We propose a novel method to evaluate ensemble diversity for object detector ensembles, coined as focal diversity, including pairwise and non-pairwise diversity metrics.

Formally, an input image $\x$ may contain multiple ground-truth objects. Instead of treating each image $\x$ as an instance, we treat each ground-truth object $\boldg$ as an instance and count the corresponding number of correctly detected objects and wrong (or missed) objects (negative instances) in ensemble diversity calculation.
We first define pairwise and non-pairwise focal negative correlations. Given an ensemble team $T$ of size $N$ and a focal model $M_i^f \in T$, the focal negative correlations are computed on the negative instances by the focal model $M_i^f$ from the validation set. The \textbf{pairwise focal negative correlation} is defined as
\begin{equation}
\small
\begin{aligned}
\lambda_{p,q}(M_p,M_q, M_i^f) &= \frac{Z^{10}+Z^{01}}{Z^{11}+Z^{10}+Z^{01}+Z^{00}} \\
\lambda_{focal}^{pair} (T, N, M_i^f) &= \frac{2}{N(N-1)} \sum_{\mathclap{M_p, M_q \in T, p < q \rule{0pt}{1.5ex}}} \lambda_{p,q}(M_p, M_q, M_i^f)
\end{aligned}
\end{equation}
where $\lambda_{p,q}(M_p,M_q,M_i^f)$ measures the level of disagreement of a pair of member models, $M_p$ and $M_q$, using the ratio of the number of instances on which one detector is correct while the other is wrong ($Z^{10}$ and $Z^{01}$) over the total number of instances~\cite{binarydisagreement}. $Z^{11}$ indicates the number of instances on which both detectors are correct, and $Z^{00}$ denotes the number of instances on which both detectors are wrong. $\lambda_{focal}^{pair} (T, N, M_i^f)$ is the pairwise focal negative correlation score.
Alternatively, the \textbf{non-pairwise focal negative correlation}, $\lambda_{focal}^{nonpair} (T, N, M_i^f)$, can be directly calculated on an ensemble of over two member models. Let $Y$ denote a random variable representing the fraction of $N$ models that made detection errors on a random ground truth object $\boldg$. $p_i$ denotes the probability of $Y=i/N$, that is the probability of $i$ out of $N$ detectors failing on a random object $\boldg$. Hence, we can compute the non-pairwise focal negative correlation as
\begin{equation}
\small
\lambda_{focal}^{nonpair} (T, N, M_i^f) = 1 - p(2)/p(1)
\end{equation}
where $p(1) = \sum_{i=1}^{N}(i/N)p_i$ and $p(2)=\sum_{i=1}^{N}(i/N)((i-1)/(N-1))p_i$. The maximum diversity of 1 is reached when one detector fails and the other detector is correct for randomly picked two detectors, that is $p(2) = 0$. The minimum diversity of 0 corresponds to $p(1)=p(2)$, where the probability for two randomly picked detectors failing is the same as the probability of one detector failing.

An ensemble team $T$ of size $N$ has $N$ focal negative correlation scores, that is one $\lambda_{focal}(T, N, M_i^f)$ for each focal (member) model $M_i^f \in T$. The ensemble team has high focal ensemble diversity if all $N$ focal negative correlations are high. Therefore, we can measure the \textbf{focal diversity}, denoted as $d_{focal}$, by using weighted average of the $N$ normalized focal negative correlation scores as
\begin{equation}
\small
d_{focal}(T, N) = \frac{1}{N} \sum_{M_i^f \in T} w_i \hat{\lambda}_{focal}(T, N, M_i^f)
\end{equation}
where $w_i$ is the weight and $\hat{\lambda}_{focal}(T, N, M_i^f)$ is the normalized focal negative correlation scores as suggested by~\cite{EnsembleBenchCVPR}.
The focal diversity metrics can be leveraged for identifying high quality ensembles of object detectors with low correlations among member models, which can effectively improve the overall ensemble robustness.

\end{document}